\documentclass[nohyperref]{article}

\usepackage{microtype}
\usepackage{graphicx}
\usepackage{subfigure}
\usepackage{booktabs} % for professional tables

\usepackage{hyperref}

\usepackage[accepted]{icml2022}

\usepackage{amsmath}
\usepackage{amssymb}
\usepackage{mathtools}
\usepackage{amsthm}
\usepackage{hyperref}
\usepackage{url}
\usepackage{epsfig}
\usepackage{multirow}

\usepackage[capitalize,noabbrev]{cleveref}

\theoremstyle{plain}
\newtheorem{theorem}{Theorem}[section]

\newtheorem{lemma}[theorem]{Lemma}
\newtheorem{corollary}[theorem]{Corollary}
\theoremstyle{definition}
\newtheorem{definition}[theorem]{Definition}

\theoremstyle{remark}

\usepackage[textsize=tiny]{todonotes}

\icmltitlerunning{}

\begin{document}

\twocolumn[
\icmltitle{Dynamically Stable Poincar\'e Embeddings for Neural Manifolds}

% It is OKAY to include author information, even for blind
% submissions: the style file will automatically remove it for you
% unless you've provided the [accepted] option to the icml2022
% package.

% List of affiliations: The first argument should be a (short)
% identifier you will use later to specify author affiliations
% Academic affiliations should list Department, University, City, Region, Country
% Industry affiliations should list Company, City, Region, Country

% You can specify symbols, otherwise they are numbered in order.
% Ideally, you should not use this facility. Affiliations will be numbered
% in order of appearance and this is the preferred way.
\icmlsetsymbol{equal}{*}

\begin{icmlauthorlist}
\icmlauthor{Jun Chen}{to}
\icmlauthor{Yuang Liu}{to}
\icmlauthor{Xiangrui Zhao}{to}
\icmlauthor{Mengmeng Wang}{to}
\icmlauthor{Yong Liu}{to}
\end{icmlauthorlist}

\icmlaffiliation{to}{Institute of Cyber-Systems and Control, Zhejiang University, Hangzhou, China}

\icmlcorrespondingauthor{Jun Chen}{junc@zju.edu.cn}
\icmlcorrespondingauthor{Yong Liu}{yongliu@iipc.zju.edu.cn}

% You may provide any keywords that you
% find helpful for describing your paper; these are used to populate
% the "keywords" metadata in the PDF but will not be shown in the document
%\icmlkeywords{Machine Learning, ICML}

\vskip 0.3in
]

% this must go after the closing bracket ] following \twocolumn[ ...

% This command actually creates the footnote in the first column
% listing the affiliations and the copyright notice.
% The command takes one argument, which is text to display at the start of the footnote.
% The \icmlEqualContribution command is standard text for equal contribution.
% Remove it (just {}) if you do not need this facility.

\printAffiliationsAndNotice{}  % leave blank if no need to mention equal contribution
%\printAffiliationsAndNotice{\icmlEqualContribution} % otherwise use the standard text.

\begin{abstract}
  In a Riemannian manifold, the Ricci flow is a partial differential equation for evolving the metric to become more regular. We hope that topological structures from such metrics may be used to assist in the tasks of machine learning. However, this part of the work is still missing. In this paper, we propose Ricci flow assisted Eucl2Hyp2Eucl neural networks that bridge this gap between the Ricci flow and deep neural networks by mapping neural manifolds from the Euclidean space to the dynamically stable Poincar\'e ball and then back to the Euclidean space. As a result, we prove that, if initial metrics have an $L^2$-norm perturbation which deviates from the Hyperbolic metric on the Poincar\'e ball, the scaled Ricci-DeTurck flow of such metrics smoothly and exponentially converges to the Hyperbolic metric. Specifically, the role of the Ricci flow is to serve as naturally evolving to the stable Poincar\'e ball. For such dynamically stable neural manifolds under the Ricci flow, the convergence of neural networks embedded with such manifolds is not susceptible to perturbations. And we show that Ricci flow assisted Eucl2Hyp2Eucl neural networks outperform with their all Euclidean counterparts on image classification tasks.
\end{abstract}

\section{Introduction}
\label{intro}

In the field of machine learning, Euclidean embeddings for representation learning are the universal and successful method, which benefits from simply convenience and closed-form formulas in the Euclidean space $\mathbb{R}^n$ endowed with the Euclidean metric $g^E$. Moreover, in the broad application of neural networks, Euclidean embeddings are also showing off, including image classification~\citep{krizhevsky2012imagenet,simonyan2014very}, semantic segmentation~\citep{long2015fully,chen2014semantic}, object detection~\citep{girshick2015fast}, etc.

\begin{figure}[tbhp]
	\centering
	\subfigure
	{\includegraphics[width=.4\textwidth]{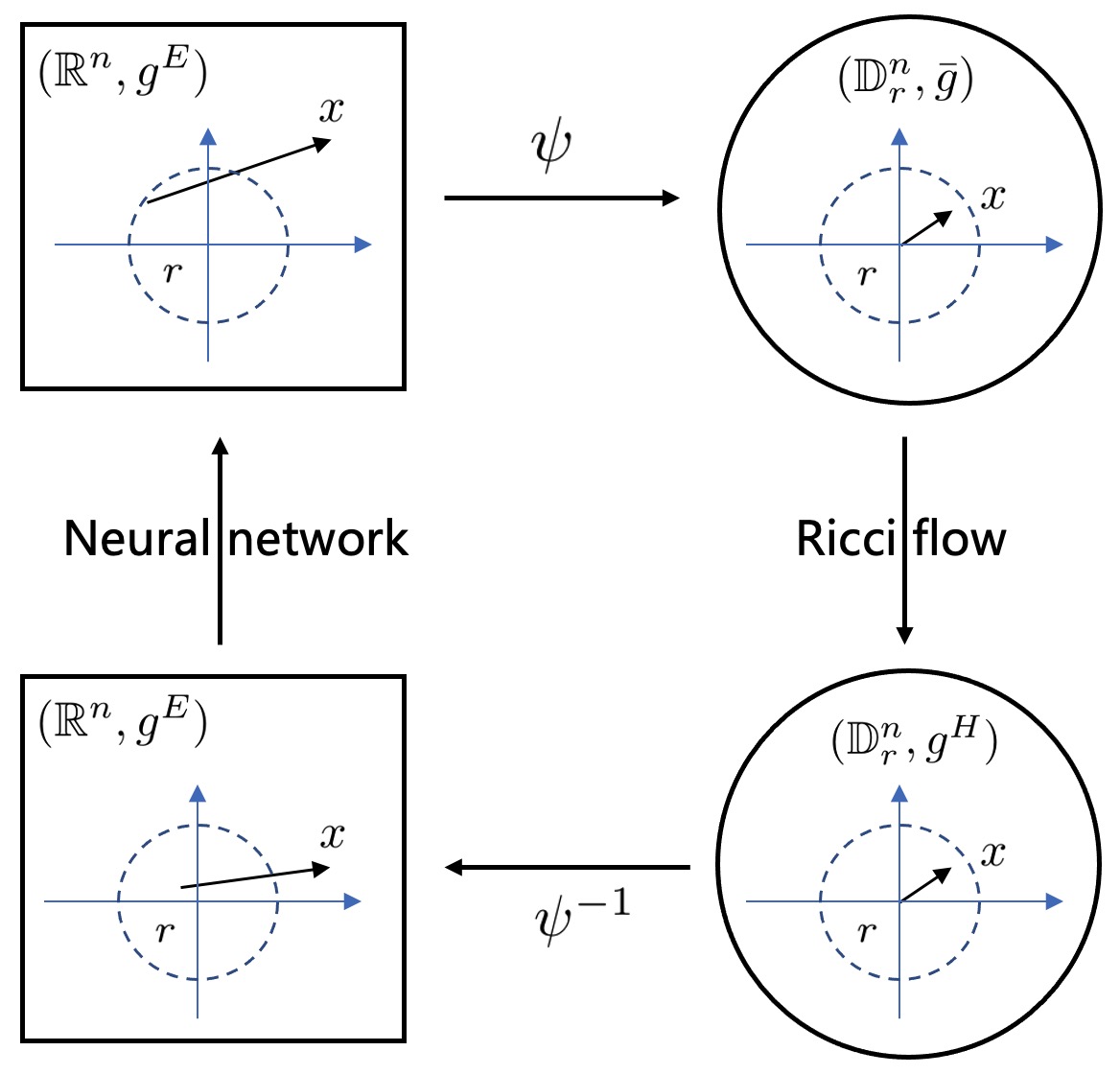}}
	\caption{An illustration of Ricci flow assisted Eucl2Hyp2Eucl neural networks. The map $\psi$ fulfils $\psi(\psi^{-1})=\psi^{-1}(\psi)=\operatorname{id}$.}
	\label{fig1}
\end{figure}

Recently, some studies have shown that the latent non-Euclidean structure in many data will affect the representation of Euclidean embeddings~\citep{bronstein2017geometric}. In such cases, the advantages of Hyperbolic embeddings are highlighted. Firstly, the Hyperbolic space provides more powerful or meaningful geometrical representations than the Euclidean space. Secondly, for the weight initialization of a neural network, the distribution is generally Gaussian~\citep{glorot2010understanding,he2015delving}, which is a manifold of constant negative curvature~\citep{amari2016information}, i.e., the Hyperbolic space~\citep{amari1987differential}. Obviously, the adoption of Hyperbolic embeddings in neural networks and deep learning has become very attractive.

For hierarchical, taxonomic or entailment data, Hyperbolic embeddings outperformed Euclidean embeddings in machine learning~\citep{sala2018representation,ganea2018hyperbolic}. For tree structure, the Euclidean space with infinite dimensions cannot be embedded with arbitrary distortion, but the Hyperbolic space with only 2 dimensions can preserve their metric~\citep{sala2018representation}. As for basic operations (e.g. matrix addition, matrix-vector multiplication, etc.) in the Hyperbolic space, Ganea et al. gave appropriate deep learning tools~\citep{ganea2019hyperbolic}. 

Despite the successful application of Hyperbolic embeddings in deep neural networks, there are two other research areas that have not yet been involved, i.e., how to {\bf avoid the perturbation} of the parameter update on the stability of the Hyperbolic space and how to {\bf combine the common advantages} of the Euclidean space and Hyperbolic space. In this paper, we propose Ricci flow assisted Eucl2Hyp2Eucl neural networks by mapping neural manifolds from the Euclidean space to the dynamically stable Poincar\'e ball and then back to the Euclidean space. Specifically, the Ricci flow can naturally make the Riemannian manifold converge by evolving metrics and provide nonlinearity for alternating Euclidean and Hyperbolic spaces (the lack of Ricci flow causes Eucl2Hyp2Eucl neural networks to be indistinguishable from their all Euclidean counterparts).

For the Ricci flow, Ye has considered stability of negatively curved manifolds on compact spaces~\citep{ye1993ricci}. Suneeta has investigated linearised stability of the Hyperbolic space under the Ricci flow~\citep{suneeta2009investigating}. Li et al. have proven stability of the Hyperbolic space in dimensions $n \geq 6$ when the deviation of the curvature of the initial metric from Hyperbolic space exponentially decays~\citep{li2010stability}. Based on the above work, Schn{\"u}rer et al. yielded stability of the Hyperbolic space in dimensions $n \geq 4$ when the initial metric is close to the Hyperbolic metric~\citep{schnurer2010stability}. A series of works have shown that, in certain cases, the Ricci flow can be used to eliminate the influence of a perturbation in the Hyperbolic space.

There are two main contributions in this paper. {\bf On the one hand}, we prove that, when initial metrics $\bar{g}$ is close to the Hyperbolic metric $g^H$ on the Poincar\'e ball (Definition~\ref{def2}), the scaled Ricci-DeTurck flow is exponential convergence for an $L^2$-norm perturbation (Lemma~\ref{lem3}). Apparently, such results are very meaningful, which shows that the application of the Ricci flow can help Poincar\'e embedded neural network manifolds to eliminate an $L^2$-norm perturbation for the metric. The Ricci flow guarantees the dynamical stability of neural manifolds during the training of neural networks with respect to the input data with transforms.

{\bf On the other hand}, furthermore, we propose Ricci flow assisted Eucl2Hyp2Eucl neural networks that have been in an alternate state between the Euclidean space $\mathbb{R}^n$ and the dynamically stable Poincar\'e ball $\mathbb{D}_{r}^{n}$. The illustration is shown in Figure~\ref{fig1}. Specifically, we map the neural network (for the output before the \textbf{softmax}) from the Euclidean space $(\mathbb{R}^n,g^E)$ to the Poincar\'e ball $(\mathbb{D}_{r}^{n},\bar{g})$ with an $L^2$-norm perturbation. Then, the Ricci flow is used to exponentially converge to the Poincar\'e ball $(\mathbb{D}_{r}^{n},g^H)$ without an $L^2$-norm perturbation. Finally, we map the neural network from the Poincar\'e ball $(\mathbb{D}_{r}^{n},g^H)$ to the Euclidean space $(\mathbb{R}^n,g^E)$. In general, Eucl2Hyp2Eucl neural networks take into account the common advantages of Euclidean and Poincar\'e embeddings, i.e., simply convenience and meaningful geometrical representations. The algorithm is shown in Algorithm~\ref{alg}.

The rest of this paper is organized as follows. Section~\ref{section2} summarizes basic works on Ricci flow. The proofs of the convergence of the Poincar\'e ball under Ricci flow are presented in Section~\ref{section3}. In Section~\ref{section4}, we yield the illustration of Ricci flow assisted Eucl2Hyp2Eucl neural networks. For the performance on classification tasks, we compare Ricci flow assisted neural networks with their all Euclidean counterparts in Section~\ref{section5}. The conclusion is given in Section~\ref{section6}.

\section{Ricci Flow}
\label{section2}

For a Riemannian manifold $\mathcal{M}$ with metric $g_0$, the Ricci flow was introduced by Hamilton to prove Thurston’s Geometrization Conjecture and consequently the Poincar\'e Conjecture~\citep{hamilton1982three}. The Ricci flow is a partial differential equation that evolves the metric:
\begin{equation}
\begin{aligned}
\frac{\partial}{\partial t} g(t) &=-2 \operatorname{Ric}(g(t)) \\
g(0) &=g_{0}
\end{aligned}
\label{ricci}
\end{equation}
where $g(t)$ is a time-dependent Riemannian metric and $\operatorname{Ric}()$ denotes the Ricci curvature tensor whose definition can be found in Appendix~\ref{app1}.

The idea is to try to evolve the Riemannian metric in some way that make the manifold become more regular, which can also be understood as rounder from the topological structure point of view. This continuous process is known as manifold ``surgery".

\subsection{Short Time Existence}

If the Ricci flow is strongly parabolic, there exists a unique solution for a short time.
\begin{theorem}
	\label{thm1}
	When $u: \mathcal{M}\times[0, T) \rightarrow \mathcal{E}$ is a time-dependent section of the vector bundle $\mathcal{E}$ where $\mathcal{M}$ is some Riemannian manifold, if the system of the Ricci flow is strongly parabolic at $u_0$ then there exists a solution on some time interval $[0, T)$, and the solution is unique for as long as it exists.
\end{theorem}
\begin{proof}
	The proofs can be found in \citep{ladyzhenskaia1988linear}.
\end{proof}

\begin{definition}
	\label{def1}
	The Ricci flow is strongly parabolic if there exists $\delta > 0$ such that for all covectors $\varphi \neq 0$ and all symmetric $h_{ij}=\frac{\partial g_{ij}(t)}{\partial t} \neq 0$, the principal symbol of $-2\operatorname{Ric}$ satisfies
	\[
	\begin{aligned}
	&[-2\operatorname{Ric}](\varphi)(h)_{ij} h^{ij} \\
	&=g^{p q}\left(\varphi_{p} \varphi_{q} h_{i j}+\varphi_{i} \varphi_{j} h_{p q}-\varphi_{q} \varphi_{i} h_{j p}-\varphi_{q} \varphi_{j} h_{i p}\right) h^{i j} \\
	&>\delta \varphi_{k} \varphi^{k} h_{r s} h^{r s}.
	\end{aligned}
	\]
\end{definition}
Since the above inequality cannot always be satisfied, the Ricci flow is not strongly parabolic. Empirically, one can not use Theorem~\ref{thm1} to prove the existence of the solution directly.

To understand which parts have an impact on its non-parabolic, one linearizes the Ricci curvature tensor.
\begin{lemma}
	\label{lem1}
	The linearization of $-2\operatorname{Ric}$ can be rewritten as
	\begin{equation}
	\begin{aligned}
	&D[-2 \operatorname{Ric}](h)_{i j}=g^{p q} \nabla_{p} \nabla_{q} h_{i j}+\nabla_{i} V_{j}+\nabla_{j} V_{i}+O(h_{ij}) \\
	&\operatorname{where} \;\;\;\;V_{i}=g^{p q}\left(\frac{1}{2} \nabla_{i} h_{p q}-\nabla_{q} h_{p i}\right).
	\end{aligned}
	\end{equation}
\end{lemma}
\begin{proof}
	The proofs can be found in Appendix~\ref{app21}.
\end{proof}

In particular, the term $O(h_{ij})$ will have no contribution to the principal symbol of $-2 \operatorname{Ric}$. For convenience of our discussion, we just ignore this term. By carefully observing the above equation, one finds that the impact on the non-parabolic of the Ricci flow comes from the terms in $V$, not the term $g^{p q} \nabla_{p} \nabla_{q} h_{i j}$. The solution is followed by the DeTurck Trick~\citep{deturck1983deforming} that has a time-dependent reparameterization of the manifold:
\begin{equation}
\begin{aligned}
\frac{\partial}{\partial t}\bar{g}(t)&=-2 \operatorname{Ric}(\bar{g}(t))- \mathcal{L}_{\frac{\partial \varphi(t)}{\partial t}} \bar{g}(t) \\
\bar{g}(0) &=\bar{g}_0,
\end{aligned}
\label{deturck}
\end{equation}
See Appendix~\ref{app22} for details. By choosing $\frac{\partial \varphi(t)}{\partial t}$ to cancel the effort of the terms in $V$, the reparameterized Ricci flow is strongly parabolic. Thus, one gets that the Ricci-DeTurck flow~\footnote{Based on \citep{sheridan2006hamilton}, we have
	\[
	\left(\mathcal{L}_{\frac{\partial \varphi(t)}{\partial t}} \bar{g}(t)\right)_{ij}= - \nabla_i W_j - \nabla_j W_i.
	\]
	Therefore, we obtain another expression of the Ricci-DeTurck flow
	\[
	\begin{aligned}
	\frac{\partial}{\partial t}\bar{g}(t)&=-2 \operatorname{Ric}(\bar{g}(t)) + \nabla_i W_j + \nabla_j W_i \\
	\bar{g}(0) &=\bar{g}_0, \quad \text { where } W_{i}=g^{p q} g_{i j}\left(\Gamma_{p q}^{j}-\tilde{\Gamma}_{p q}^{j}\right).
	\end{aligned}
	\]
} has a unique solution for a short time.

\subsection{Curvature Explosion at Singularity}

Subsequently, we will present the behavior of the Ricci flow in finite time and show that the evolution of the curvature tends to develop singularities. Before giving the core demonstration, Theorem~\ref{thm4}, some foreshadowing proofs need to be prepared.

\begin{theorem}
	\label{thm2}
	Given a smooth Riemannian metric $g_0$ on a closed manifold $\mathcal{M}$, there exists a maximal time interval $[0, T)$ such that a solution $g(t)$ of the Ricci flow, with $g(0) = g_0$, exists and is smooth on $[0, T)$, and this solution is unique.
\end{theorem}
\begin{proof}
	The proofs can be found in \citep{sheridan2006hamilton}.
\end{proof}

\begin{theorem}
	\label{thm3}
	Let $\mathcal{M}$ be a closed manifold and $g(t)$ a smooth time-dependent metric on $\mathcal{M}$, defined for $t \in [0, T)$. If there exists a constant $C < \infty$ for all $x \in \mathcal{M}$ such that
	\begin{equation}
	\int_{0}^{T}\left|\frac{\partial}{\partial t} g_x(t)\right|_{g(t)} d t  \leq C,
	\end{equation}
	then the metrics $g(t)$ converge uniformly as $t$ approaches $T$ to a continuous metric $g(T)$ that is uniformly equivalent to $g(0)$ and satisfies
	\[
	e^{-C} g_x(0) \leq g_x(T) \leq e^C g_x(0).
	\]
\end{theorem}
\begin{proof}
	The proofs can be found in Appendix~\ref{app23}.
\end{proof}

\begin{corollary}
	\label{cor1}
	Let $(\mathcal{M}, g(t))$ be a solution of the Ricci flow on a closed manifold. If $|\operatorname{Rm}|_{g(t)}$ is bounded on a finite time $[0, T)$, then $g(t)$ converges uniformly as $t$ approaches $T$ to a continuous metric
	$g(T)$ which is uniformly equivalent to $g(0)$.
\end{corollary}
\begin{proof}
	The bound on $|\operatorname{Rm}|_{g(t)}$ implies one on $|\operatorname{Ric}|_{g(t)}$. Based on Eq.~\ref{ricci}, we can extend the bound on $|\frac{\partial}{\partial t}g(t)|_{g(t)}$. Therefore, we obtain an integral of a bounded quantity over a finite interval is also bounded, by Theorem~\ref{thm3}.
\end{proof}

\begin{theorem}
	\label{thm4}
	If $g_0$ is a smooth metric on a compact manifold $\mathcal{M}$, the Ricci flow with $g(0) = g_0$
	has a unique solution $g(t)$ on a maximal time interval $t\in [0, T)$. If $T < \infty$, then
	\begin{equation}
	\lim _{t \rightarrow T}\left(\sup _{x \in \mathcal{M}}|\operatorname{Rm}_x(t)|\right)=\infty.
	\end{equation}
\end{theorem}
\begin{proof}
	For a contradiction, we assume that $|\operatorname{Rm}_x(t)|$ is bounded by a constant. It follows from Corollary~\ref{cor1} that the metrics $g(t)$ converge uniformly in the norm induced by $g(t)$ to a smooth metric $g(T)$. Based on Theorem~\ref{thm2}, it is possible to find a solution to the Ricci flow on $t \in [0, T)$ because the smooth metric $g(T)$ is uniformly equivalent to initial metric $g(0)$.
	
	Hence, one can extend the solution of the Ricci flow after the time point $t=T$, which is the result for continuous derivatives at $t=T$. Naturally, the time $T$ of existence of the Ricci flow has not been maximal, which contradicts our assumption. In other words, $|\operatorname{Rm}_x(t)|$ is unbounded.
\end{proof}

According to Theorem~\ref{thm4}, the Riemann curvature $|\operatorname{Rm}|_{g(t)}$ becomes divergent and tends to explode, as approaching the singular time $T$.

\section{The Poincar\'e Ball under Ricci Flow}
\label{section3}

\subsection{Basics of Hyperbolic Space and The Poincar\'e Ball}

The Hyperbolic space has several isometric models~\citep{anderson2006hyperbolic}. In this paper, similarly as in \citep{nickel2017poincare} and \citep{ganea2018hyperbolic} , we choose an $n$-dimensional Poincar\'e ball $\mathbb{D}_{r}^{n}$ with radius $1/\sqrt{r}$. 

Empirically, the Poincar\'e ball can be defined by the background manifold $\mathbb{D}_{r}^{n}:=\left\{x \in \mathbb{R}^{n} \mid r\|x\|^{2}<1\right\}$ endowed with the Hyperbolic metric:
\begin{equation}
g_{x}^{H}=\lambda_{x}^{2} g^{E}, \quad \text { where } \lambda_{x}:=\frac{2}{1-r\|x\|^{2}}.
\label{hyperbolic}
\end{equation}
Note that the Euclidean metric $g^{E}$ is equal to the identity matrix $\boldsymbol{I}$. For $r>0$, $\mathbb{D}_{r}^{n}$ denotes the open ball of radius $1/\sqrt{r}$. In particular, if $r=0$, then one recovers the Euclidean space $\mathbb{D}_{0}^{n}=\mathbb{R}^{n}$.

By Corollary~\ref{cor2}, The Riemannian gradient endowed with $g_{x}^{H}$ for any point $x \in \mathbb{D}^{n}$ is known to be given by~\footnote{Note that the Riemannian gradient is similar to the natural gradient~\citep{martens2015optimizing,martens2020new} in Riemannian manifold defined by the KL divergence.}
\begin{equation}
\partial_{x}^{H}=\frac{1}{\lambda_{x}^{2}} \partial^{E}, \quad \text { where } \lambda_{x}:=\frac{2}{1-r\|x\|^{2}}.
\end{equation}

\begin{corollary}
	\label{cor2}
	In view of information geometry~\citep{amari2016information}, the steepest descent direction in a Riemannian manifold endowed with $g$ satisfies
	\begin{equation}
	\partial^{g}=g^{-1} \partial^{E},
	\end{equation}
	with respect to the steepest descent direction $\partial^{E}$ in Euclidean space.
\end{corollary}
\begin{proof}
	The proofs can be found in Appendix~\ref{app4}.
\end{proof}

\subsection{The Hyperbolic Metric}

As the Hyperbolic space evolves under Ricci flow, it is convenient to consider the rescaled Ricci-DeTurck flow~\citep{schnurer2010stability} by Eq.~\ref{deturck} 
\begin{equation}
\begin{aligned}
\frac{\partial}{\partial t}\bar{g}(t)&=-2 \operatorname{Ric}(\bar{g}(t)) + \nabla_i W_j + \nabla_j W_i - 2(n-1)\bar{g}(t) \\
\bar{g}(0) &=\bar{g}_0, \quad \text { where } W_{i}=g^{p q} g_{i j}\left(\sideset{^{\bar{g}}}{^j_{pq}}{\mathop{\mathrm{\Gamma}}}-\sideset{^{g^H}}{^j_{pq}}{\mathop{\mathrm{\Gamma}}}\right).
\label{rescale}
\end{aligned}
\end{equation}
The Hyperbolic metric $g^H$ on $\mathbb{D}^n_r$ of sectional curvature $-r$ is a stationary point to Eq.~\ref{rescale}.

Subsequently, we will discuss that the Poincar\'e ball $(\mathbb{D}_{r}^{n},\bar{g}(t))$ uniformly converges to the Poincar\'e ball $(\mathbb{D}_{r}^{n},g^H)$ under the rescaled Ricci-DeTurck flow when the given perturbation satisfies Definition~\ref{def2}. Obviously, the introduction of the Ricci flow can ensure dynamically stable Poincar\'e ball.

\begin{definition}
	\label{def2}
	Let $\bar{g}$ be a metric on $\mathbb{D}_{r}^{n}$. There exists a $\epsilon>0$ such that
	\[
	(1+\epsilon)^{-1} g^H \leq \bar{g} \leq(1+\epsilon) g^H,
	\]
	which can be said that $\bar{g}$ is $\epsilon$-close to $g^H$.
\end{definition} 

\subsection{Finite Time Existence}

We denote the norm of a tensor as $|\cdot|$, then

\begin{lemma}
	\label{lem2}
	Given a Riemannian metric $\bar{g}(t)$ on $\mathbb{D}_{r}^{n}$, there exists a maximal time interval $(0,T)$ such that a solution to Eq.~\ref{rescale} exists and is smooth~\footnote{Smooth is equivalent to $C^\infty$, i.e., any derivative is continuous.} on $(0,T)$. Specifically, $\bar{g}(t)$ is $\epsilon$-close to $g^H$ of sectional curvature $-r$. If a small enough $\epsilon>0$, then there exists
	\begin{equation}
	\begin{aligned}
	\frac{\partial}{\partial t}\left|\bar{g}-g^H\right|^{2} \leq &\Delta\left|\bar{g}-g^H\right|^{2}-2|\nabla(\bar{g}-g^H)|^{2} \\
	&+4\left|\bar{g}-g^H\right|^{2}
	\end{aligned}
	\end{equation}
	where $\Delta$ is the Laplacian defined by $\bar{g}^{i j} \nabla_{i} \nabla_{j}$.
\end{lemma}
\begin{proof}
	The proofs can be found in Appendix~\ref{app3}.
\end{proof}

\begin{corollary}
	\label{cor3}
	Given the Poincar\'e ball $\mathbb{D}_{r}^{n}$ where $\partial \mathbb{D}_{r}^{n}$ denotes the boundary, there exists a maximal interval $[0,T)$ such that a solution $\bar{g}(t)$ on $\mathbb{D}_{r}^{n}$ to Eq.~\ref{rescale} exists and is smooth on $[0,T)$. Specifically, $\bar{g}|_{\partial\mathbb{D}_{r}^{n}}=g^H|_{\partial\mathbb{D}_{r}^{n}}$. There exists a constant $C>0$ such that
	\begin{equation}
	\sup_{\mathbb{D}_{r}^{n}}\left|\bar{g}-g^H\right| \leq C.
	\end{equation}
\end{corollary}
\begin{proof}
	As long as Definition~\ref{def2} is satisfied, Lemma~\ref{lem2} gives the proofs.
\end{proof}

We denote the $L^{\infty}$-norm with respect to the Hyperbolic metric $g^H$ as $\|\cdot\|_{L^{\infty}}$, then

\begin{theorem}
	\label{thm5}
	For a solution $\bar{g}(t)$ on $\mathbb{D}_{r}^{n}$ to Eq.~\ref{rescale} that exists and is smooth on a maximal time interval $[0,T)$, if $\bar{g}(0)$ is a metric on $\mathbb{D}_{r}^{n}$ satisfying $\|\bar{g}(0)-g^H\|_{L^{\infty}}\leq\epsilon$ where $\epsilon>0$, then there exists a constant $C$ such that
	\begin{equation}
	\|\bar{g}(t)-g^H\|_{L^{\infty}}\leq C\|\bar{g}(0)-g^H\|_{L^{\infty}}\leq\epsilon \cdot C.
	\end{equation}
\end{theorem}
\begin{proof}
	The proof follows the similar statement \citep{simon2002deformation,bamler2010stability}.
\end{proof}

Empirically, Theorem~\ref{thm5} yields that the solutions to the rescaled Ricci-DeTurck flow exists in finite time. Otherwise, Corollary~\ref{cor3} gives an upper bound on $\left|\bar{g}-g^H\right|$, which allows us to integrate it.

\subsection{Exponential Convergence}

\begin{theorem}
	\label{thm6}
	Based on Corollary~\ref{cor3}, we further have
	\begin{equation}
	\int_{\mathbb{D}_{r}^{n}}\left|\bar{g}(t)-g^H\right|^{2} d \Omega \leq e^{-(A(n,r)-4) t} \int_{\mathbb{D}_{r}^{n}}\left|\bar{g}(0)-g^H\right|^{2} d \Omega
	\end{equation}
	where $\Omega$ is the volume element with respect to $\mathbb{D}_{r}^{n}$.
\end{theorem}
\begin{proof}
	Using Lemma~\ref{lem2}, we yield
	\[
	\begin{aligned}
	&\frac{\partial}{\partial t}\int_{\mathbb{D}_{r}^{n}}\left|\bar{g}-g^H\right|^{2} d \Omega \leq 4\int_{\mathbb{D}_{r}^{n}}\left|\bar{g}-g^H\right|^{2} d \Omega \\
	&+\int_{\mathbb{D}_{r}^{n}}\bar{g}^{i j} \nabla_{i} \nabla_{j}\left|\bar{g}-g^H\right|^2 -2\left|\nabla(\bar{g}-g^H)\right|^2 d \Omega \\
	&=4\int_{\mathbb{D}_{r}^{n}}\left|\bar{g}-g^H\right|^{2} d \Omega-2\int_{\mathbb{D}_{r}^{n}}\left|\nabla(\bar{g}-g^H)\right|^2 d \Omega \\
	&+\int_{\mathbb{D}_{r}^{n}}\nabla_{i}(\bar{g}^{i j}  \nabla_{j})\left|\bar{g}-g^H\right|^2 -\nabla_{i} \bar{g}^{i j}\nabla_{j}\left|\bar{g}-g^H\right|^2 d \Omega.
	\end{aligned}
	\]
	In the second step, we use $\nabla_{i}(\bar{g}^{i j}  \nabla_{j})\left|\bar{g}-g^H\right|^2= \nabla_{i} \bar{g}^{i j}\nabla_{j}\left|\bar{g}-g^H\right|^2+\bar{g}^{i j} \nabla_{i} \nabla_{j}\left|\bar{g}-g^H\right|^2$.
	
	As $|\bar{g}-g^H|_{\partial\mathbb{D}_{r}^{n}}=0$, we compute, using Stokes theorem~\citep{wald2010general},
	\begin{equation}
	\begin{aligned}
	&\int_{\mathbb{D}_{r}^{n}}\nabla_{i}(\bar{g}^{i j}  \nabla_{j})\left|\bar{g}-g^H\right|^2 d \Omega \\
	&= \int_{\partial\mathbb{D}_{r}^{n}}n_{i}\bar{g}^{i j}  \nabla_{j}\left|\bar{g}-g^H\right|^2 dS=0
	\end{aligned}
	\end{equation}
	where $dS$ is the area element with respect to $\partial\mathbb{D}_{r}^{n}$ and $n_i$ is the outer normal vector with respect to $\partial\mathbb{D}_{r}^{n}$.
	We define
	\begin{equation}
	A(n,r) = \inf \frac{\int_{\mathbb{D}_{r}^{n}}2\left|\nabla(\bar{g}-g^H)\right|^2+\nabla_{i} \bar{g}^{i j}\nabla_{j}\left|\bar{g}-g^H\right|^2 d \Omega}{\int_{\mathbb{D}_{r}^{n}}\left|\bar{g}-g^H\right|^{2} d \Omega},
	\end{equation}
	then
	\begin{equation}
	\frac{\partial}{\partial t}\int_{\mathbb{D}_{r}^{n}}\left|\bar{g}-g^H\right|^{2} d \Omega \leq (4 - A(n,r))\int_{\mathbb{D}_{r}^{n}}\left|\bar{g}-g^H\right|^{2} d \Omega.
	\end{equation}
	In view of differential equation, the above inequality extends, using $\mathcal{F}(t)=\int_{\mathbb{D}_{r}^{n}}\left|\bar{g}(t)-g^H\right|^{2} d \Omega$ and $\mathcal{F}(0)=\int_{\mathbb{D}_{r}^{n}}\left|\bar{g}(0)-g^H\right|^{2} d \Omega$,
	\[
	\begin{aligned}
	&\int \frac{\partial \mathcal{F}(t)}{\mathcal{F}(t)} \leq (4 - A(n,r))\int \partial t \\
	& \rightarrow \log \mathcal{F}(t) \leq (4 - A(n,r))t + \log \mathcal{F}(0) \\
	& \rightarrow \mathcal{F}(t) \leq e^{-(A(n,r)-4)t} \mathcal{F}(0).
	\end{aligned}
	\]
	Based on Theorem~\ref{thm5}, we have $A(n,r)\geq4$ because $\left|\bar{g}-g^H\right|^2$ decays.
\end{proof}

\begin{lemma}
	\label{lem3}
	Based on Theorem~\ref{thm6}, we yield the estimate
	\begin{equation}
	\left\|\bar{g}(t)-g^H\right\|^{2}_{L^2(\mathbb{D}_{r}^{n})}  \leq e^{-(A(n,r)-4) t} \left\|\bar{g}(0)-g^H\right\|^{2}_{L^2(\mathbb{D}_{r}^{n})}
	\end{equation}
	where $\|\cdot\|_{L^{2}}$ is the $L^{2}$-norm with respect to the Hyperbolic metric $g^H$.
\end{lemma}
\begin{proof}
	The proofs follow directly from Theorem~\ref{thm5} and Corollary~\ref{cor3}.
\end{proof}

Consequently, we see that the scaled Ricci-DeTurck flow is exponential convergence for an $L^2$-norm perturbation.

\section{Ricci Flow Assisted Neural Networks}
\label{section4}

\subsection{Ricci Curvature in Neural Networks}

On the one hand, a neural network is trained on the given dataset and its geometrical structure will gradually become regular. On the other hand, the Ricci flow is a process of ``surgery" on a manifold, which will make the manifold also become regular. There seems to be a common goal of these two evolutionary methods. In this way, we can embed a Riemannian manifold into the neural network, and utilize the Ricci flow to assist in the training of neural networks on dynamically stable manifolds.

Based on the previous section, we have embedded the Poincar\'e ball into a neural network. And initial metrics $\bar{g}$ for an $L^2$-norm perturbation that deviates from the Hyperbolic metric $g^H$ will converge to $g^H$ under the Ricci flow. That is excellent because we embed a dynamically stable Poincar\'e ball for neural networks, which will not affect the convergence of neural networks.

Empirically, for each input, we can embed an $n$-dimensional Poincar\'e ball into the output of neural networks before the \textbf{softmax}. Apparently, time-dependent metrics $\bar{g}(t)$ corresponding to the output can be well-defined. 

Now, Let us see the Ricci curvature of neural networks. According to Eq.~\ref{ricci}, the tensor $-2\operatorname{Ric}(\bar{g}(t))$ approaches zero, as $\frac{\partial}{\partial t} \bar{g}(t)$ approaches zero. We can yield, referring to Appendix~\ref{app1},
\begin{equation}
\begin{aligned}
&-2\operatorname{Ric}(\bar{g})=-2R_{i j k}^{i} \\
&=\bar{g}^{i p}\left(\partial_{i}\partial_{j}\bar{g}_{p k}-\partial_{i}\partial_{k}\bar{g}_{p j}+\partial_{p}\partial_{k}\bar{g}_{i j}+\partial_{p}\partial_{j}\bar{g}_{i k}\right).
\end{aligned}
\end{equation}

Inspired by~\citep{kaul2019riemannian}, we treat the term $i$ and $p$ as the translation and rotation by considering translation invariance instead of rotation invariance. As for rotations, the standard data augmentation does not include such transformations. For the fairness of ablation studies, we just exclude rotations, i.e., $\partial_{p}(\partial_{k}\bar{g}_{i j}+\partial_{j}\bar{g}_{i k})=0$. Therefore, $\partial_{k}\bar{g}$ and $\partial_{j}\bar{g}$ can be the row and column transformation respectively for the input data. Consequently, we have
\begin{equation}
\label{condition}
-2\operatorname{Ric}(\bar{g})
=\bar{g}^{i p}\partial_{i}\left(\partial_{j}\bar{g}_{pk}-\partial_{k}\bar{g}_{pj}\right).
\end{equation}

For the convenient form, we approximate partial derivatives with differences, with respect to the input translation dimensions $k$ and $j$
\begin{equation}
\partial_{k}\bar{g}=(\bar{g}|_{k1} - \bar{g}|_{k2})/(k1-k2),
\end{equation}
\begin{equation}
\partial_{j}\bar{g}=(\bar{g}|_{j1} - \bar{g}|_{j2})/(j1-j2).
\end{equation}
In general, $(k1-k2)$ and $(j1-j2)$ are translations less than 4 pixels, which is consistent with data augmentation.

\subsection{Mutual Mapping of Euclidean Space and The Poincar\'e Ball}

We consider alternating neural manifolds between Euclidean embeddings and Poincar\'e embeddings in back-propagation, which greatly retains the common advantages of the Euclidean space and Hyperbolic space. 

Firstly, we map the neural manifold from the Euclidean space $(\mathbb{R}^n,g^E)$ to the Poincar\'e ball $(\mathbb{D}_r^{n},\bar{g})$, where $\bar{g}$ is a solution to the Ricci flow. By adding the regularization to the neural network, the tensor $-2\operatorname{Ric}(\bar{g})$ approaches zero to make $\bar{g}$ satisfy the Definition~\ref{def2}. Secondly, we perform the Ricci flow for evolving the metric $\bar{g}$ to the Hyperbolic metric $g^H$. Thirdly, we map the neural manifold from the Poincar\'e ball $(\mathbb{D}_r^{n},g^H)$ to the Euclidean space $(\mathbb{R}^n,g^E)$. Fourthly, we complete the backpropagation of the gradient for the neural network.

Since the Poincar\'e ball is conformal to the Euclidean space, we give the exponential and logarithmic maps~\footnote{In~\citep{ganea2019hyperbolic}, the exponential map and logarithm map are also used as mutual mapping of the Euclidean space and the Poincar\'e ball, i.e., $\exp_r: \mathbb{R}^{n} \rightarrow \mathbb{D}_r^{n}$ and $\log_r: \mathbb{D}_r^{n} \rightarrow \mathbb{R}^{n}$.}:
\begin{lemma}
	\label{lem4}
	With respect to the origin $(x=0)$, the exponential map $\exp_r: T_x  \mathbb{D}_r^{n} \rightarrow \mathbb{D}_r^{n}$ and the logarithmic map $\log_r: \mathbb{D}_r^{n} \rightarrow T_x  \mathbb{D}_r^{n}$ are given for $\mu \in T_x  \mathbb{D}_r^{n} \backslash \{\mathbf{0}\}$ and $\nu \in \mathbb{D}_r^{n} \backslash \{\mathbf{0}\}$ by:
	\begin{equation}
	\exp _{r}(\mu)=\tanh (\sqrt{r}\|\mu\|) \frac{\mu}{\sqrt{r}\|\mu\|},
	\label{exp}
	\end{equation}
	\begin{equation}
	\log _{r}(\nu)=\tanh ^{-1}(\sqrt{r}\|\nu\|) \frac{\nu}{\sqrt{r}\|\nu\|}.
	\label{log}
	\end{equation}
\end{lemma}
\begin{proof}
	The proofs are followed with \citep{ganea2018hyperbolic}. The algebraic check concludes the identity $\log_{r}(\exp _{r}(\mu))=\mu$.
\end{proof}

The overall process is illustrated in Figure~\ref{fig1} where the map $\psi$ and $\psi^{-1}$ are the exponential map and logarithm map respectively.

\subsection{Ricci Flow Assisted Eucl2Hyp2Eucl Neural Networks}

The above discussion leaves the foreshadowing for designing neural networks. The precise computation of neural networks with $l$ layers is performed as follows:
\begin{equation}
\label{neural}
\begin{aligned}
&x=f(a;\theta,b), \\
&f(a;\theta,b)=\sigma_l[\cdots\sigma_2(\sigma_1(a\theta_1+b_1)\theta_2+b_2)+\cdots+b_l], \\
&y=\operatorname{softmax}(x),
\end{aligned}
\end{equation}
where $a$ is the input, $\sigma$ is a nonlinear ``activation" function, $\theta$ is the weight and $b$ is the bias.

In addition to requiring the neural networks to converge, we also require that $\bar{g}$ is $\epsilon$-close to $g^H$ based on Definition~\ref{def2}, which is the necessary condition for the evolution of the Ricci flow. Obviously, we may achieve the goal by adding a regularization into loss function. Followed by Eq.~\ref{condition}, we yield the regularization
\begin{equation}
\label{norm}
\sideset{^{\bar{g}}}{}{\mathop{N}}=\left\|\frac{\bar{g}|_{k1} - \bar{g}|_{k2}}{k1-k2}-\frac{\bar{g}|_{l1} - \bar{g}|_{l2}}{l1-l2}\right\|^{2}_{L^2}.
\end{equation}

Combined with Definition~\ref{def2}, the upper bound of Eq.~\ref{norm} is estimated by
\begin{equation}
\label{upper}
\begin{aligned}
&\sideset{^{\bar{g}}}{}{\mathop{N}} \\
&\leq\left\|\frac{(1+\epsilon)^2 g^H|_{k1} - g^H|_{k2}}{(1+\epsilon)(k1-k2)}-\frac{g^H|_{l1} - (1+\epsilon)^2 g^H|_{l2}}{(1+\epsilon)(l1-l2)}\right\|^{2}_{L^2} \\
&=\frac{g^{E}}{1+\epsilon}\left\|\frac{(1+\epsilon)^2 \lambda^2_{x_{k1}} - \lambda^2_{x_{k2}}}{k1-k2}-\frac{\lambda^2_{x_{l1}} - (1+\epsilon)^2 \lambda^2_{x_{l2}}}{l1-l2}\right\|^{2}_{L^2},
\end{aligned}
\end{equation}
and the lower bound of Eq.~\ref{norm} is estimated by
\begin{equation}
\label{lower}
\begin{aligned}
&\sideset{^{\bar{g}}}{}{\mathop{N}} \\
&\geq\left\|\frac{g^H|_{k1} - (1+\epsilon)^2 g^H|_{k2}}{(1+\epsilon)(k1-k2)}-\frac{(1+\epsilon)^2g^H|_{l1} - g^H|_{l2}}{(1+\epsilon)(l1-l2)}\right\|^{2}_{L^2} \\
&=\frac{g^{E}}{1+\epsilon}\left\|\frac{\lambda^2_{x_{k1}} - (1+\epsilon)^2 \lambda^2_{x_{k2}}}{k1-k2}-\frac{(1+\epsilon)^2\lambda^2_{x_{l1}} - \lambda^2_{x_{l2}}}{l1-l2}\right\|^{2}_{L^2}.
\end{aligned}
\end{equation}

As the evolution of the Ricci flow approaches to converge, the estimate of Eq.~\ref{norm} tends to be stable:
\begin{equation}
\label{estimate}
\sideset{^{\bar{g}}}{}{\mathop{N}}\stackrel{\operatorname{Ricci\ flow}}{\longrightarrow}N=\left\|\frac{\lambda^2_{x_{k1}} - \lambda^2_{x_{k2}}}{k1-k2}-\frac{\lambda^2_{x_{l1}} - \lambda^2_{x_{l2}}}{l1-l2}\right\|^{2}_{L^2}.
\end{equation}

Consequently, we divide Ricci flow assisted Eucl2Hyp2Eucl neural manifold evolution into two stages: coarse convergence and fine convergence. With the help of the regularization $N$, the corresponding metric of this neural manifold will converge to the neighbourhood of $g^H$, and then the Ricci flow will complete the final convergence. Each training of the neural network includes these two stages, therefore, Ricci flow assisted Eucl2Hyp2Eucl neural networks are trained on dynamically stable Poincar\'e embeddings as shown in Algorithm~\ref{alg}.

\begin{algorithm}[htpb]
	\caption{For a gradient update of Ricci flow assisted Eucl2Hyp2Eucl neural networks, we choose 4 different translations: $k1$, $k2$, $j1$ and $j2$. We define a target $z$ and a balancing parameter $\alpha(\epsilon)$. For $x$ embeded in the Euclidean space, one uses $\partial^E$ as the gradients. Otherwise, for $x$ embeded in the Poincar\'e ball with $g^H$ and $\bar{g}$, one uses $\partial_{x}^{H}$ and $\bar{g}^{-1}\partial^E$ as the gradients based on Corollary~\ref{cor2}.}
	\label{alg}
	\begin{algorithmic}[1]
		\STATE \{Inference\}
		\FOR{$m$ {\bfseries in} $k1,k2,j1,j2$}
		\STATE $x_m=f(a_m;\theta,b)$ based on Eq.~\ref{neural}
		\STATE $\sideset{^{\bar{g}}}{_{m}}{\mathop{x}}=\exp _{r}(x_m)$ based on Eq.~\ref{exp}
		\ENDFOR
		\STATE \{Regularization\}
		\STATE $\sideset{^{\bar{g}}}{}{\mathop{N}}=\left\|\frac{\bar{g}|_{k1} - \bar{g}|_{k2}}{k1-k2}-\frac{\bar{g}|_{l1} - \bar{g}|_{l2}}{l1-l2}\right\|^{2}_{L^2}$ based on Eq.~\ref{norm}
		\STATE $\sideset{^{\bar{g}}}{_{k1}}{\mathop{x}} \rightarrow \sideset{^{g^H}}{_{k1}}{\mathop{x}}$ by computing the Ricci flow ($\bar{g} \rightarrow g^H$) based on Eq.~\ref{ricci}
		\STATE $N=\left\|\frac{\lambda^2_{x_{k1}} - \lambda^2_{x_{k2}}}{k1-k2}-\frac{\lambda^2_{x_{l1}} - \lambda^2_{x_{l2}}}{l1-l2}\right\|^{2}_{L^2}$ based on Eq.~\ref{estimate}
		\STATE \{Loss\}
		\STATE $x_{k1}=\log_{r}(\sideset{^{g^H}}{_{k1}}{\mathop{x}})$ based on Eq.~\ref{log}
		\STATE $y=\operatorname{softmax}(x_{k1})$
		\STATE $\operatorname{loss}=\|y-z\|^2_{L^2}+\alpha(\epsilon) \cdot N$
	\end{algorithmic}
\end{algorithm}

\section{Experiment}
\label{section5}

{\bf CIFAR datasets.} The two CIFAR datasets \cite{krizhevsky2009learning} consist of natural color images with 32$\times$32 pixels, respectively 50,000 training and 10,000 test images, and we hold out 5,000 training images as a validation set from the training set. CIFAR10 consists of images organized into 10 classes and CIFAR100 into 100 classes. We adopt a standard data augmentation scheme (random corner cropping and random flipping) that is widely used for these two datasets. We normalize the images using the channel means and standard deviations in preprocessing.

{\bf Settings.} We set total training epochs as 200 where the learning rate of each parameter group is set as a cosine annealing schedule. The learning strategy is a weight decay of 0.0005, a batch size of 128, SGD optimization. On CIFAR10 and CIFAR100 datasets, we apply ResNet18~\citep{he2016deep}, ResNet50~\citep{he2016deep}, VGG11~\citep{simonyan2014very} and MobileNetV2~\citep{sandler2018mobilenetv2} to test the classification accuracy. All experiments are conducted for 5 times, and the statistics of the last 10/5 epochs' test accuracy are reported as a fair comparison.

{\bf Details.} For Ricci flow assisted Eucl2Hyp2Eucl neural networks and all Euclidean neural networks, we both use the same training strategy and network structure. Note that we both train neural networks from scratch with the initialization Xavier~\citep{glorot2010understanding}.

\subsection{Classification Tasks}

In this experiment, we compare the classification accuracy of Ricci flow assisted Eucl2Hyp2Eucl neural networks and all Euclidean neural networks on CIFAR datasets. As Table~\ref{table1} shows~\footnote{Note that Ricci flow assisted Eucl2Hyp2Eucl neural networks are only used in the training, and we also use all Euclidean neural networks in offline inference.}, our proposed Ricci flow assisted Eucl2Hyp2Eucl neural networks has better performance than all Euclidean neural networks. 
Meanwhile, compared to CIFAR10 dataset, the improvement on CIFAR100 dataset seem to be more remarkable. We conjecture that more complex classification tasks bring more meaningful geometric structures to the neural network, and Ricci flow assisted Eucl2Hyp2Eucl neural networks can just mine these geometric representations as much as possible.

Note that experiment on ImageNet can be found in Appendix~\ref{app5}.

\begin{table*}[htbp]
	\caption{The classification accuracy results on CIFAR datasets with ResNet18, ResNet50, VGG11 and MobileNetV2.}
	\begin{center}
		\begin{tabular}{c|c|c|c}
			\toprule[1pt]
			\textbf{Network} & \textbf{Method} & \textbf{CIFAR10} & \textbf{CIFAR100} \\
			\toprule[1pt]
			\multirow{2}{*}{ResNet18} & All Euclidean Neural Network & 95.25$\pm$0.13\% & 77.25$\pm$0.15\% \\
			& Ricci Flow Assisted Eucl2Hyp2Eucl Neural Network & 95.73$\pm$0.09\% & 77.87$\pm$0.12\% \\
			\midrule
			\multirow{2}{*}{ResNet50} & All Euclidean Neural Network & 95.58$\pm$0.17\% & 78.47$\pm$0.33\% \\
			& Ricci Flow Assisted Eucl2Hyp2Eucl Neural Network & 96.01$\pm$0.06\% & 79.79$\pm$0.28\% \\
			\midrule
			\multirow{2}{*}{VGG11} & All Euclidean Neural Network & 92.28$\pm$0.16\% & 71.66$\pm$0.21\% \\
			& Ricci Flow Assisted Eucl2Hyp2Eucl Neural Network & 93.02$\pm$0.12\% & 73.28$\pm$0.27\% \\
			\midrule
			\multirow{2}{*}{MobileNetV2} & All Euclidean Neural Network & 92.28$\pm$0.25\% & 72.33$\pm$0.31\% \\
			& Ricci Flow Assisted Eucl2Hyp2Eucl Neural Network & 93.75$\pm$0.22\% & 73.42$\pm$0.21\% \\
			\toprule[1pt]
		\end{tabular}
	\end{center}
	\label{table1}
\end{table*}

\subsection{Metrics Evolution Analysis}

For the training of Ricci flow assisted Eucl2Hyp2Eucl neural networks, we hope to observe the evolution of neural manifolds by the change of metrics. Meanwhile, as far as we define the metric $\bar{g}(t)$, we can use the length $|ds^2|=\sqrt{\sum_{i,j} \bar{g}_{ij}(t) d\xi_i d\xi_j}$ to intuitively reflect the change of metrics. Specifically, we define a ball whose radius is equal to $|ds^2|$:
\begin{equation}
B_r(t):=\left\{r=\sqrt{\sum_{i,j} \bar{g}_{ij}(t) d\xi_i d\xi_j}\right\}.
\end{equation}

\begin{figure}[thbp]
	\centering
	\subfigure[CIFAR10]
	{\includegraphics[width=0.5\textwidth]{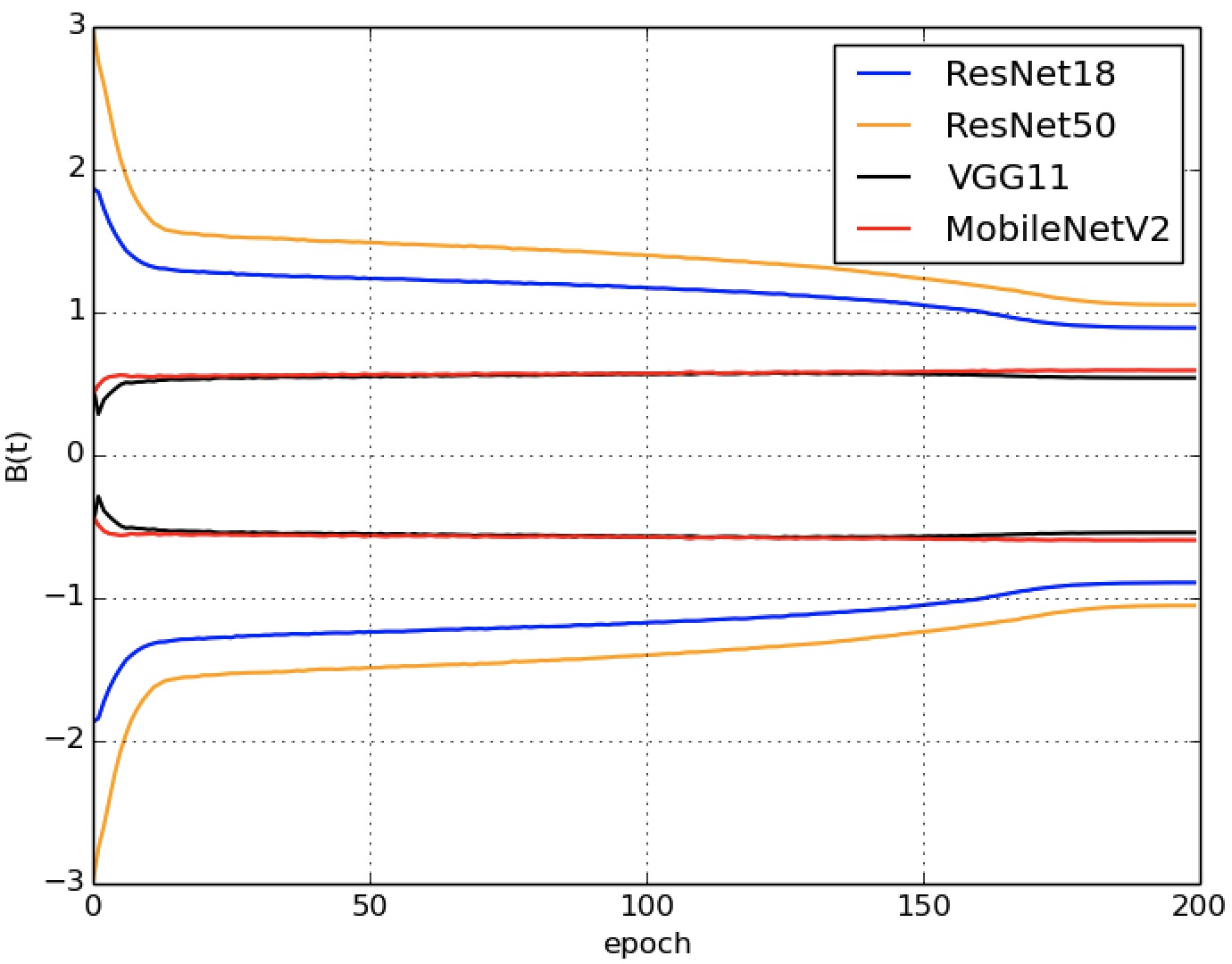}}
	\subfigure[CIFAR100]
	{\includegraphics[width=0.5\textwidth]{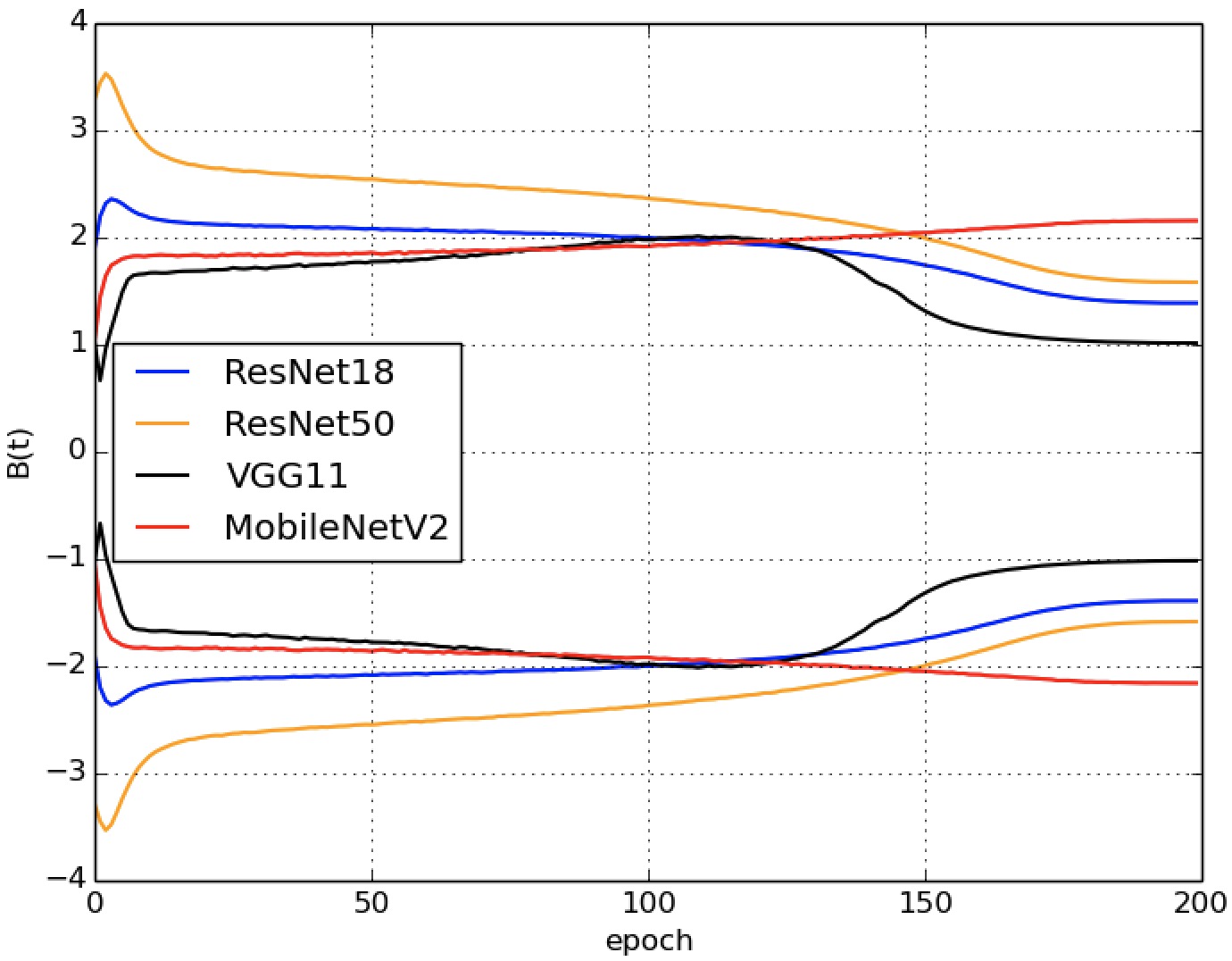}}
	\caption{The evolution of metrics $\bar{g}(t)$ by the radius of a ball with the epoch of training process.}
	\label{curvature}
\end{figure}

By observing the change of the ball in Figure~\ref{curvature}, we can know the change of the metric. Through simple observation, metrics $\bar{g}(t)$ have a rapid convergence at the beginning of training, and then become relatively flat. It seems that the convergence behavior of metrics is affected by the network structure rather than the depth (There has the similar evolution behavior on ResNet18 and ResNet50). In the training, experiments show that all metrics for Ricci flow assisted Eucl2Hyp2Eucl neural networks have a stable convergence. It is consistent with the evolution of scaled Ricci-DeTurck flow in Section~\ref{section3}.

\subsection{Training Time Analysis}

Our hardware environment is conducted with an Intel(R) Xeon(R) E5-2650 v4 CPU(2.20GHz), GeForce GTX 1080Ti GPU. We test the training time per iteration as for Ricci flow assisted Eucl2Hyp2Eucl neural networks and all Euclidean neural networks with ResNet18 in CIFAR10. For finishing one iteration of training, Ricci flow assisted Eucl2Hyp2Eucl neural network costs $\textbf{81.06s}$ and all Euclidean neural network costs $\textbf{39.76s}$.

\section{Conclusion}
\label{section6}

Ricci flow assisted Eucl2Hyp2Eucl neural networks not only provide the simply convenience and closed-form formula in the Euclidean space (for offline inference), but also take into account the geometric representation of the Hyperbolic space. This provides a new idea for neural networks to obtain meaningful geometric representations. Empirically, we found that Ricci flow assisted neural networks outperform with their all Euclidean counterparts on classification tasks.

The Ricci flow plays a vital role in Eucl2Hyp2Eucl neural networks, not only eliminating an $L^2$-norm perturbation of the Hyperbolic metric, but also acting as a smooth evolution from the Euclidean space to the Poincar\'e ball. Actually, Eucl2Hyp2Eucl neural networks without Ricci flow will become the same as all Euclidean neural networks. We hope that this paper will open an exciting future direction which will use the Ricci flow to assist neural network training in a dynamically stable manifold.

% In the unusual situation where you want a paper to appear in the
% references without citing it in the main text, use \nocite
\nocite{langley00}

\bibliography{example_paper}
\bibliographystyle{icml2022}

%%%%%%%%%%%%%%%%%%%%%%%%%%%%%%%%%%%%%%%%%%%%%%%%%%%%%%%%%%%%%%%%%%%%%%%%%%%%%%%
%%%%%%%%%%%%%%%%%%%%%%%%%%%%%%%%%%%%%%%%%%%%%%%%%%%%%%%%%%%%%%%%%%%%%%%%%%%%%%%
% APPENDIX
%%%%%%%%%%%%%%%%%%%%%%%%%%%%%%%%%%%%%%%%%%%%%%%%%%%%%%%%%%%%%%%%%%%%%%%%%%%%%%%
%%%%%%%%%%%%%%%%%%%%%%%%%%%%%%%%%%%%%%%%%%%%%%%%%%%%%%%%%%%%%%%%%%%%%%%%%%%%%%%
\newpage
\appendix
\onecolumn
\section{Differential Geometry}
\label{app1}
1. Riemann curvature tensor (Rm) is a (1,3)-tensor defined for a 1-form $\omega$:
\[
R^l_{ijk}\omega_l=\nabla_i \nabla_j \omega_k - \nabla_j \nabla_i \omega_k
\]
where the covariant derivative of $F$ satisfies
\[
\begin{aligned}
\nabla_{p} F_{i_{1} \ldots i_{k}}^{j_{1} \ldots j_{l}}=&\partial_{p} F_{i_{1} \ldots i_{k}}^{j_{1} \ldots j_{l}}+\sum_{s=1}^{l} F_{i_{1} \ldots i_{k}}^{j_{1} \ldots q \ldots j_{l}} \Gamma_{p q}^{j_{s}} -\sum_{s=1}^{k} F_{i_{1} \ldots q_{\ldots} i_{k}}^{j_{1} \ldots j_{l}} \Gamma_{p i_{s}}^{q}.
\end{aligned}
\]

In particular, coordinate form of the Riemann curvature tensor is:
\[
R_{i j k}^{l}=\partial_{i} \Gamma_{j k}^{l}-\partial_{j} \Gamma_{i k}^{l}+\Gamma_{j k}^{p} \Gamma_{i p}^{l}-\Gamma_{i k}^{p} \Gamma_{j p}^{l}.
\]

2. Christoffel symbol in terms of an ordinary derivative operator is:
\[
\Gamma^k_{ij}=\frac{1}{2}g^{kl}(\partial_i g_{jl}+\partial_j g_{il}-\partial_l g_{ij}).
\]

3. Ricci curvature tensor (Ric) is a (0,2)-tensor:
\[
R_{ij}=R^p_{pij}.
\]

4. Scalar curvature is the trace of the Ricci curvature tensor:
\[
R=g^{ij}R_{ij}.
\]

5. Lie derivative of $F$ in the direction $\frac{d \varphi(t)}{dt}$:
\[
\mathcal{L}_{\frac{d \varphi(t)}{dt}} F=\left(\frac{d}{dt}\varphi^*(t) F\right)_{t=0}
\]
where $\varphi(t): \mathcal{M} \rightarrow \mathcal{M}$ for $t\in(-\epsilon,\epsilon)$ is a time-dependent diffeomorphism of $\mathcal{M}$ to $\mathcal{M}$.

\section{Proof for the Ricci Flow}
\label{app2}

\subsection{Proof for Lemma~\ref{lem1}}
\label{app21}

\begin{definition}
	\label{linear}
	The linearization of the Ricci curvature tensor is given by
	\[
	D[\operatorname{Ric}](h)_{i j}=-\frac{1}{2} g^{p q}(\nabla_{p} \nabla_{q} h_{i j}+\nabla_{i} \nabla_{j} h_{p q}-\nabla_{q} \nabla_{i} h_{jp}-\nabla_{q} \nabla_{j} h_{i p}).
	\]
\end{definition}
\begin{proof}
	Based on Appendix~\ref{app1}, we have
	\[
	\nabla_{q} \nabla_{i} h_{j p} =\nabla_{i} \nabla_{q} h_{j p}-R_{q i j}^{r} h_{r p}-R_{q i p}^{r} h_{j m}.
	\]
	Combining with Definition~\ref{linear}, we can obtain the deformation equation because of $\nabla_k g_{ij}=0$,
	\[
	\begin{aligned}
	D[-2 \mathrm{Ric}](h)_{i j}=& g^{p q} \nabla_{p} \nabla_{q} h_{i j}+\nabla_{i}\left(\frac{1}{2} \nabla_{j} h_{p q}-\nabla_{q} h_{j p}\right) +\nabla_{j}\left(\frac{1}{2} \nabla_{i} h_{p q}-\nabla_{q} h_{i p}\right)+O(h_{ij}) \\
	=& g^{p q} \nabla_{p} \nabla_{q} h_{i j}+\nabla_{i} V_{j}+\nabla_{j} V_{i}+O(h_{ij}).
	\end{aligned}
	\]
\end{proof}

\subsection{Description of the DeTurck Trick}
\label{app22}

Using a time-dependent diffeomorphism $\varphi(t)$, we express the pullback metrics $\bar{g}(t)$:
\[
g(t)=\varphi^*(t) \bar{g}(t)
\]
is a solution of the Ricci flow. Based on the chain rule for the Lie derivative in Appendix~\ref{app1}, we can calculate
\[
\begin{aligned}
\frac{\partial}{\partial t} g(t) &=\frac{\partial\left(\varphi^{*}(t) \bar{g}(t)\right)}{\partial t} \\
&=\left(\frac{\partial\left(\varphi^{*}(t+\tau) \bar{g}(t+\tau)\right)}{\partial \tau}\right)_{\tau=0} \\
&=\left(\varphi^{*}(t) \frac{\partial \bar{g}(t+\tau)}{\partial \tau}\right)_{\tau=0}+\left(\frac{\partial\left(\varphi^{*}(t+\tau) \bar{g}(t)\right)}{\partial \tau}\right)_{\tau=0} \\
&=\varphi^{*}(t) \frac{\partial}{\partial t}\bar{g}(t)+\varphi^{*}(t) \mathcal{L}_{\frac{\partial \varphi(t)}{\partial t}} \bar{g}(t).
\end{aligned}
\]
With the help of Equation~(\ref{ricci}), for the reparameterized metric, we have
\[
\begin{aligned}
\frac{\partial}{\partial t} g(t)&=\varphi^{*}(t) \frac{\partial}{\partial t}\bar{g}(t)+\varphi^{*}(t) \mathcal{L}_{\frac{\partial \varphi(t)}{\partial t}} \bar{g}(t) \\
&=-2 \operatorname{Ric}(\varphi^*(t) \bar{g}(t)) \\
&=-2 \varphi^*(t) \operatorname{Ric}(\bar{g}(t)).
\end{aligned}
\]
The diffeomorphism invariance of the Ricci curvature tensor is used in the last step. The above equation is equivalent to
\[
\frac{\partial}{\partial t}\bar{g}(t)=-2 \operatorname{Ric}(\bar{g}(t))- \mathcal{L}_{\frac{\partial \varphi(t)}{\partial t}} \bar{g}(t).
\]

\subsection{Proof for Theorem~\ref{thm3}}
\label{app23}

Considering any $x \in \mathcal{M}$, $t_0 \in [0, T)$, $V \in T_x \mathcal{M}$, we have
\[
\begin{aligned}
\left|\log \left(\frac{g_x (t_0)(V, V)}{g_x (0)(V, V)}\right)\right| &=\left|\int_{0}^{t_{0}} \frac{\partial}{\partial t}\left[\log g_x (t)(V, V)\right] d t\right| \\
&=\left|\int_{0}^{t_{0}} \frac{\frac{\partial}{\partial t} g_x (t)(V, V)}{g_x (t)(V, V)} d t\right| \\
& \leq \int_{0}^{t_{0}}\left|\frac{\partial}{\partial t} g_x (t) \left(\frac{V}{|V|_{g(t)}}, \frac{V}{|V|_{g(t)}}\right)\right| d t \\
& \leq \int_{0}^{t_{0}}\left|\frac{\partial}{\partial t} g_x (t)\right|_{g(t)} d t \\
& \leq C.
\end{aligned}
\]
By exponentiating both sides of the above inequality, we have
\[
e^{-C} g_x(0)(V, V) \leq g_x(t_0)(V, V) \leq e^C g_x(0)(V, V).
\]
This inequality can be rewritten as
\[
e^{-C} g_x(0) \leq g_x(t_0)(V, V) \leq e^C g_x(0)(V, V)
\]
because it holds for any $V$. Thus, the metrics $g(t)$ are uniformly equivalent to $g(0)$.

Now, we have the well-defined integral:
\[
g_x(T) - g_x(0) = \int_{0}^{T}\frac{\partial}{\partial t} g_x (t) d t.
\]
We say that this integral is well-defined because of two reasons. Firstly, as long as the metrics are smooth, the integral exists. Secondly, the integral is absolutely integrable. Based on the norm inequality induced by $g(0)$, one has
\[
|g_x(T) - g_x(t)|_{g(0)} \leq \int_{t}^{T}\left|\frac{\partial}{\partial t} g_x (t)\right|_{g(0)} d t.
\]
For each $x \in \mathcal{M}$, the above integral will approach to zero as $t$ approaches $T$. Because $\mathcal{M}$ is compact, the metrics $g(t)$ converge uniformly to a continuous metric $g(T)$ which is uniformly equivalent to $g(0)$ on $\mathcal{M}$. Moreover, we can show that
\[
e^{-C} g_x(0) \leq g_x(T) \leq e^C g_x(0).
\]

\section{Proof for Lemma~\ref{lem2}}
\label{app3}

Empirically, we have

\[
\frac{\partial}{\partial t}\left|\bar{g}-g^H\right|^{2} = \sum \frac{\partial}{\partial t} \left(\bar{g}^2_{ij}-2\bar{g}_{ij}g^H_{ij}\right).
\]
At this time, $\bar{g}_{ij}$ or $g^H_{ij}$ is the given coordinate representation. Based on \citep{shi1989deforming}, we can write $\frac{\partial}{\partial t} \bar{g}_{i j}$:

\[
\begin{aligned}
\frac{\partial}{\partial t} \bar{g}_{i j}=&\bar{g}^{a b} \nabla_{a} \nabla_{b} \bar{g}_{i j} \\
&+2 \bar{g}_{i j}\left(\bar{g}^{k l}\left(g^H_{k l}-\bar{g}_{k l}\right)\right)+2\left(\bar{g}_{i j}-g^H_{i j}\right) \\
&+\frac{1}{2} \bar{g}^{a b} \bar{g}^{p q}\left(\nabla_{i} \bar{g}_{p a} \nabla_{j} \bar{g}_{q b}+2 \nabla_{a} \bar{g}_{j p} \nabla_{q} \bar{g}_{i b}\right. \\
&\left.-2 \nabla_{a} \bar{g}_{j p} \nabla_{b} g_{i q}
-2 \nabla_{j} \bar{g}_{p a} \nabla_{b} \bar{g}_{i q}-2 \nabla_{i} \bar{g}_{p a} \nabla_{b} \bar{g}_{j q}\right).
\end{aligned}
\]

In the third line of the above formula, we denote contractions as in \citep{hamilton1982three}:
\[
\begin{aligned}
&(\nabla \bar{g} * \nabla \bar{g})_{ij} =\bar{g}^{a b} \bar{g}^{p q}\left(\nabla_{i} \bar{g}_{p a} \nabla_{j} \bar{g}_{q b}+2 \nabla_{a} \bar{g}_{j p} \nabla_{q} \bar{g}_{i b}\right. \\
&\left.-2 \nabla_{a} \bar{g}_{j p} \nabla_{b} g_{i q}
-2 \nabla_{j} \bar{g}_{p a} \nabla_{b} \bar{g}_{i q}-2 \nabla_{i} \bar{g}_{p a} \nabla_{b} \bar{g}_{j q}\right).
\end{aligned}
\]
Followed by Definition~\ref{def2}, we have the estimate
\[
\begin{aligned}
& \frac{\partial}{\partial t}\left|\bar{g}-g^H\right|^{2}=2 \sum\left(\bar{g}_{i j}-g^H_{i j}\right)\left(\frac{\partial}{\partial t} \bar{g}_{ij}\right) \\
=& \bar{g}^{i j} \nabla_{i} \nabla_{j}\left|\bar{g}-g^H\right|^{2}-2|\nabla \bar{g}|^{2} \\
&+2 \sum\left(\bar{g}_{ij}-g^H_{ij}\right)\left[2\left(\bar{g}_{ij}-g^H_{ij}\right)-2 \bar{g}_{ij} \sum\left(\bar{g}^{kl}\left(\bar{g}_{kl}-g^H_{kl}\right)\right)\right] \\
&+\sum\left(\bar{g}_{ij}-g^H_{ij}\right)(\nabla \bar{g} * \nabla \bar{g})_{ij} \\
= & \bar{g}^{i j} \nabla_{i} \nabla_{j}\left|\bar{g}-g^H\right|^{2}-2|\nabla(\bar{g}-g^H)|^{2} \\
&+4 \left|\bar{g}-g^H\right|^2 - 4\sum\left[\left(\bar{g}_{ij}-g^H_{ij}\right) \bar{g}_{ij} \sum\left(\bar{g}^{kl}\left(\bar{g}_{kl}-g^H_{kl}\right)\right)\right] \\
\leq & \bar{g}^{i j} \nabla_{i} \nabla_{j}\left|\bar{g}-g^H\right|^{2}-2|\nabla(\bar{g}-g^H)|^{2} +4 \left|\bar{g}-g^H\right|^2 \\
& +\frac{4\epsilon}{1+\epsilon}\sum\left[g^H_{ij} \bar{g}_{ij}\sum\left(\bar{g}^{kl}\left(\bar{g}_{kl}-g^H_{kl}\right)\right)\right].
\end{aligned}
\]
In the second line, we have $\nabla_{i}\nabla_{j}(\bar{g}^2)=\nabla_{i}(2\bar{g}\cdot\nabla_{j}\bar{g})=2\nabla_{i}\bar{g}\cdot\nabla_{j}\bar{g}+2\bar{g}\nabla_{i}\nabla_{j}\bar{g}$ and $\nabla_{i}\nabla_{j}({g^H}^2)=0$ because $\nabla$ is compatible with $g^H$.

\section{Proof for Corollary~\ref{cor2}}
\label{app4}

\begin{definition}
	\label{divergence}
	$D[P:Q]$ is called a divergence when it satisfies the following criteria:
	
	1) $D[P:Q] \geq 0$.
	
	2) $D[P:Q]=0$ when and only when $P=Q$.
	
	3) When $P$ and $Q$ are sufficiently close, by denoting their coordinates by $\boldsymbol{\xi}_P$ and $\boldsymbol{\xi}_Q=\boldsymbol{\xi}_P+d\boldsymbol{\xi}$, the Taylor expansion of $D$ is written as
	\[
	D[\boldsymbol{\xi}_P:\boldsymbol{\xi}_P+d\boldsymbol{\xi}]=\frac{1}{2}\sum_{i,j} g_{ij}(\boldsymbol{\xi}_P) d\xi_i d\xi_j + O(|d\boldsymbol{\xi}|^3),
	\]
	and Riemannian metric $g_{ij}$ is positive-definite, depending on $\boldsymbol{\xi}_P$.
\end{definition}

With constraint the divergence in a constant $C$, we do the minimization of the loss function $\operatorname{loss}(\boldsymbol{\theta})$ in Lagrangian form:
\[
\begin{aligned}
&d\boldsymbol{\theta}^* = \operatorname*{arg\ min}_{s.t. D[\boldsymbol{\theta}:\boldsymbol{\theta}+d\boldsymbol{\theta}]=C} \operatorname{loss}(\boldsymbol{\theta}+d\boldsymbol{\theta}) \\
&=\operatorname*{arg\ min}_{d\boldsymbol{\theta}} \operatorname{loss}(\boldsymbol{\theta}+d\boldsymbol{\theta}) +\lambda\left(D[\boldsymbol{\theta}:\boldsymbol{\theta}+d\boldsymbol{\theta}]-C\right) \\
&\approx \operatorname*{arg\ min}_{d\boldsymbol{\theta}} \operatorname{loss}(\boldsymbol{\theta}) + \partial_{\boldsymbol{\theta}}\operatorname{loss}(\boldsymbol{\theta})^\top d\boldsymbol{\theta} +\frac{\lambda}{2}d\boldsymbol{\theta}^\top g_{ij} d\boldsymbol{\theta}-C\lambda.
\end{aligned}
\]

To solve the above minimization, we set its derivative with respect to $d\boldsymbol{\theta}$ to zero:
\[
\begin{aligned}
&0 = \partial_{\boldsymbol{\theta}}L(\boldsymbol{\theta}) + \frac{\lambda}{2} g_{ij} d\boldsymbol{\theta} \\
&-\frac{\lambda}{2} g_{ij} d\boldsymbol{\theta} = \partial_{\boldsymbol{\theta}}\operatorname{loss}(\boldsymbol{\theta}) \\
& d\boldsymbol{\theta} = -\frac{2}{\lambda} g^{-1}_{ij}\partial_{\boldsymbol{\theta}}\operatorname{loss}(\boldsymbol{\theta}).
\end{aligned}
\]
Where a constant factor $2/\lambda$ can be absorbed into learning rate. The opposite direction $-d\boldsymbol{\theta}$ is the steepest descent direction in a Riemannian manifold endowed with $g_{ij}$.
%%%%%%%%%%%%%%%%%%%%%%%%%%%%%%%%%%%%%%%%%%%%%%%%%%%%%%%%%%%%%%%%%%%%%%%%%%%%%%%
%%%%%%%%%%%%%%%%%%%%%%%%%%%%%%%%%%%%%%%%%%%%%%%%%%%%%%%%%%%%%%%%%%%%%%%%%%%%%%%

\section{ImageNet}
\label{app5}

{\bf ImageNet dataset.} ILSVRC2012 \cite{deng2009imagenet} image classification dataset consists of 1.2 million high-resolution natural images where the validation set contains 50k images. These images are organized into 1000 categories of the object for training, which are resized to 224$\times$224 pixels before fed into the network. In the next experiment, we report our single-crop evaluation results using top-1 and top-5 accuracy. In preprocessing, images are resized randomly to 256$\times$256 pixels, and then a random crop of 224$\times$224 is selected for training.

{\bf Settings.} We set total training epochs as 100 where the learning rate of each parameter group is set as a cosine annealing schedule. The learning strategy is a weight decay of 0.0001, a batch size of 128, SGD optimization. On Table~\ref{table2}, we apply ResNet18 and ResNet50 to test the classification accuracy.

\begin{table*}[htbp]
	\caption{The classification accuracy results on ImageNet dataset with ResNet18 and ResNet50.}
	\begin{center}
		\begin{tabular}{c|c|c|c}
			\toprule[1pt]
			\textbf{Network} & \textbf{Method} & \textbf{Top-1 Acc} & \textbf{Top-5 Acc} \\
			\toprule[1pt]
			\multirow{2}{*}{ResNet18} & All Euclidean Neural Network & 69.76\% & 89.08\% \\
			& Ricci Flow Assisted Eucl2Hyp2Eucl Neural Network & 70.55\% & 89.67\% \\
			\midrule
			\multirow{2}{*}{ResNet50} & All Euclidean Neural Network & 76.13\% & 92.86\% \\
			& Ricci Flow Assisted Eucl2Hyp2Eucl Neural Network & 77.41\% & 93.54\% \\
			\toprule[1pt]
		\end{tabular}
	\end{center}
	\label{table2}
\end{table*}

\end{document}